\documentclass[11pt]{article}
\usepackage{acl} % remove review here

\usepackage{times}
\usepackage{latexsym}
\usepackage[T1]{fontenc}
\usepackage{float}
\usepackage[utf8]{inputenc}
\usepackage[textsize=tiny]{todonotes}
\usepackage{microtype}
\usepackage{amsmath}
\usepackage{amsfonts}
\usepackage{inconsolata}
\usepackage{subcaption}
\usepackage{booktabs}
\usepackage{amsthm}
\usepackage{graphicx}
\usepackage{enumitem}
\newcommand{\mythanks}{\thanks{equal contribution}}

\title{Probe-Free Low-Rank Activation Intervention}
\author{Chonghe Jiang \mythanks  \\
  CUHK\\
  \texttt{chjiang@link.cuhk.edu.hk} \\\And
  Bao Nguyen \mythanks\\
  CUHK \\
  \texttt{nbnguyen@se.cuhk.edu.hk} \\\AND
  Anthony Man-Cho So \\
  CUHK \\
  \texttt{manchoso@se.cuhk.edu.hk} \\\And
  Viet Anh Nguyen  \\
  CUHK \\
  \texttt{nguyen@se.cuhk.edu.hk} \\}
\newtheorem{theorem}{Theorem}[section]

\newcommand{\ba}{\begin{array}}
	\newcommand{\ea}{\end{array}}

\newcommand{\mc}{\mathcal}
\newcommand{\Let}{=}

\def\ba{\bm{a}}

\begin{document}
\maketitle
\begin{abstract}
Language models (LMs) can produce texts that appear accurate and coherent but contain untruthful or toxic content. Inference-time interventions that edit the hidden activations have shown promising results in steering the LMs towards desirable generations. Existing activation intervention methods often comprise an activation probe to detect undesirable generation, triggering the activation modification to steer subsequent generation. This paper proposes a probe-free intervention method \textbf{FLORAIN} for all attention heads in a specific activation layer. It eliminates the need to train classifiers for probing purposes. The intervention function is parametrized by a sample-wise nonlinear low-rank mapping, which is trained by minimizing the distance between the modified activations and their projection onto the manifold of desirable content. Under specific constructions of the manifold and projection distance, we show that the intervention strategy can be computed efficiently by solving a smooth optimization problem. The empirical results, benchmarked on multiple base models, demonstrate that FLORAIN consistently outperforms several baseline methods in enhancing model truthfulness and quality across generation and multiple-choice tasks. Our implementation can be found at \href{https://github.com/nguyenngocbaocmt02/EFI}{https://github.com/nguyenngocbaocmt02/EFI}.

\end{abstract}

\section{Introduction}
Transformer-based language models (LMs) have revolutionized generative modeling for natural language processing (NLP) \citep{ref:radford2019language,ref:brown2020language,ref:anthropic2024claude,ref:jiang2023mistral}. The LM training pipeline in various NLP areas typically involves three steps: (i) provide the model with prompts that include demonstrations of the task, (ii) learn from the examples, and (iii) make predictions without additional training. In the generating process, there is evidence \citep{ref:ji2023survey,ref:rawte2023troubling,ref:xu2024hallucination} showing that LM output can contain undesirable contents, such as inaccurate answers, toxic answers, or answers without linguistic meaning.

One of the predominant solutions to the above issues is \textit{ activation intervention}~\citep{ref:subramani2022extracting,ref:hernandez2023measuring,ref:li2024inference, ref:nguyen2025risk}. Compared with previous studies on model editing \citep{ref:zhang2023survey} and supervised fine-tuning \citep{ref:li2024pre}, activation intervention alters the model activations during inference time. Therefore, it does not need to alter model weights using a subset of
text samples, requiring fewer computational resources. However, new challenges must be addressed in the activation intervention task: (\textit{Q1}) How to improve the quality of activation intervention through appropriate modeling with theoretical underpinnings? (\textit{Q2}) How to enhance the computational efficiency of activation intervention by designing a low-cost intervention strategy? 

For question (\textit{Q1}), \citet{ref:burns2022discovering} reveals latent knowledge inside the internal activations of a language model and finds a direction in activation space that satisfies logical consistency properties. 
Following this observation, several works focus on finding a good intervention systematically by two-stage methods \citep{yin2024lofit,ref:li2024inference,pham2024householder}. In the first stage, they train a classifier (the probe) on the activations of a network to identify a subset of attention heads that are most important for learning the specific task. In the second step, they propose intervention policies based on the probe and geometric transformation, e.g., linear operations \citep{ref:li2024inference}, addictive offset bias vector \citep{yin2024lofit}, linear transformation using chance constrained programming~\citep{ref:nguyen2025risk}, and householder transformation \citep{pham2024householder}.

The preliminary success of these activation intervention methods
motivates us to improve the desirable generation of LMs. Specifically, we want to design an inference activation intervention method that does not require the `detect and rectify' procedure, addressing (\textit{Q1}) and (\textit{Q2}) simultaneously. Our core idea is to depict the region of the desirable answers and consider the related low-rank transformation mapping for intervention. Based on these, the parameters in the intervention mapping are the solutions to a nonlinear low-rank optimization problem, which minimizes the distance between the intervention vector and its projection on the region. Compared with the previous region modelings of the LM activation space \citep{mamou2020emergence, janiak2024characterizing}, we give the concrete formulation of the region with an analytical projection operator under the suitable distance measure.

 \noindent \textbf{Contributions.} We propose Probe-\textbf{F}ree 
 \textbf{Lo}w-\textbf{R}ank \textbf{A}ctivation \textbf{In}tervention (FLORAIN), a novel intervention method characterized by
\begin{itemize}[leftmargin=5mm]
    \item An ellipsoid model to capture the region of desirable output. In the training phase, we extract the activation heads of desirable answers to different questions and build the ellipsoid region models. The ellipsoid model is estimated using the first- and second-order statistical information of the activations extracted from the training data. To combat the ill-conditioning of the estimate due to the low sample size, we propose an extrapolation strategy to pull the region away from the region of \textit{un}desirable outputs.
    \item A low-rank intervention mapping. We propose to minimize the gap between the post-intervention activation and its projection onto the desirable ellipsoid region. The low-rank intervention mapping is the product of a one-layer perception with another low-rank matrix variable. We show that the objective function admits an analytical form under the Mahalanobis distance projection. In addition, we show that the training problem is smooth and, thereby, can be calculated efficiently using a lightweight gradient descent method or its preconditioned version. 
\end{itemize}

From the computational aspect, our FLORAIN intervention method edits the activation vectors on \textit{one} layer of the Transformer network. This contrasts to ITI~\cite{ref:li2024inference} that intervene different heads spread out in \textit{multiple} layers of the network. Intervening in one layer, as FLORAIN does, has two significant advantages: First, it provides conditions for parallelization to reduce inference intervention time. Second, since the intervention vectors are concentrated in one layer, the strategy reduces the representation shifts by avoiding misleading intervention in the subsequent layers. 

\section{Related Work}

\textbf{Controllable generation.}  Model editing~\citep{ref:wang2023knowledge, ref:zhang2024comprehensive} alters the model parameters to control the output, making it a powerful method for controllable generation. Another important category in controllable generation is fine-tuning, which includes Supervised Fine-Tuning (SFT,~\citep{ref:peng2023instruction, ref:gunel2020supervised}) and Reinforcement Learning from Human Feedback (RLHF,~\citep{ref:ouyang2022training, ref:griffith2013policy}). These methods typically require altering model weights, incurring substantial resources and costs for computation.
 
\noindent \textbf{Activation intervention} at inference time is an emerging technique for controllable generation~\citep{ref:li2024inference,ref:singh2024mimic,yin2024lofit}. Unlike model editing and fine-tuning techniques, the inference time intervention does not require altering the model parameters, leading to cheaper computational costs. \citet{ref:li2024inference} proposes a headwise intervention method for eliciting truthful generated answers of a language model. \citet{ref:singh2024mimic} considers the optimal transport plan between two empirical distributions to carry out the intervention. LoFit~\citep{yin2024lofit} identifies a specific subset of attention heads crucial for learning a particular task. It then fine-tunes the intervention vectors in those chosen heads. Another recent work \citep{pham2024householder} considers doing intervention activation using modified householder transformation based on the linear probe framework. 

\noindent \textbf{Region Modeling in LM.}
Various works aim to reveal how the semantic features influence the `region' of embedding vectors in the transformer-based LM. The work \citep{mamou2020emergence} utilizes mean-field theoretic manifold analysis to connect the geometry of feature representations with the linear separability of classes. \citet{janiak2024characterizing} identifies stable regions in the residual stream of Transformers, where the model output remains insensitive to small activation changes but exhibits high sensitivity at the region boundaries. However, most prior works focus on specific findings in region modeling and overlook the potential of enhancing activation intervention through transport between two regions.

\noindent \textbf{Low-Rank Optimization} is widely studied in machine learning, statistics, and signal processing \citep{olikier2022low,zhanxuan2020low,cory2021new}. The motivation for formulating problems into low-rank optimization in several applications lies in two folds: the nature of the low-rank property of the ground truth and the goal for achieving lightweight complexity in algorithm design \citep{mcrae2024benign}.  
The low-rank concept has also been applied to LM model-editing, fine-tuning, and model compression \citep{hu2021lora,hsu2022language}. 

\section{Nonlinear Low-Rank Intervention Mapping}
\label{sec:low_rank}
This section defines the nonlinear low-rank activation intervention maps. In our probe-free intervention framework, this mapping serves as the computationally efficient intervention function during inference time.

For clarity, we consider a specific layer $\ell$ of the transformer network, and the index of the layer $\ell$ will be omitted in the following statement. This layer contains $H$ attention heads; each head is of dimension $d$, so the output of this layer is a $D = d \times H$ dimensional activation vector. For an input $x_i$, the activation in the layer $\ell$ is denoted by $a_{i}=[a_{i1};a_{i2};\cdots;a_{iH}] \in \mathbb{R}^{D}$. Our inference-time intervention will modify the activation vector head-wise for $a_{ih}$ for all $h \in [H]$.  

The construction of the low-rank mapping involves two key considerations: (i) the intervention strategy should be generalized to different inputs, and (ii) the intervention strategy should not change a desirable answer too much. Motivated by these requirements, we construct our efficient nonlinear low-rank intervention as follows:
\begin{subequations} \label{eq:f-all}
\begin{equation} \label{eq:f-def}
        f: a \mapsto (I + L(a) R^{\top}) a + s,
    \end{equation}
    where $L(a)\in \mathbb{R}^{D \times k}$ is a low-rank matrix that depends on the input $a$, $R \in \mathbb{R}^{D \times k}$ is a constant low-rank matrix and $s\in \mathbb{R}^{D\times 1}$ is a vector. 
    The parametrization of $L(a)$ is described as follows:
    \begin{equation} \label{eq:L}
        L_{i}(a) = \phi (W_{i} \circ a + b_i) \quad \forall i \in [k],
    \end{equation}
\end{subequations}
    where $L_i(a)$ and $W_i$ refer to the $i$-th column of $L(a)$ and $W$, respectively. The notation $\circ$ denotes the Hadamard product of two vectors. The matrix $W \in \mathbb{R}^{D \times k}$ is the weight matrix, $b \in \mathbb{R}^{D\times k}$ is the bias term, and $\phi$ is the activation function applied component-wise to the input \footnote{Note that in this context, the \textbf{activation function} refers to the concept used in neural networks and is unrelated to the \textbf{activation} intervention task.}. One candidate for this activation function is $\phi(\cdot) = \tanh(\cdot)$. This choice enhances intervention performance, as the mapping $x \mapsto \tanh(x)$ can output both negative and positive real numbers. Additionally, the output of the mapping is bounded between $(-1, 1)$, resolving the scaling issues of the bilinear terms $L(a) R^\top$ in formulation~\eqref{eq:f-def} automatically. If the components of $L(a)$ are unbounded, one can multiply $W$ by any positive constant $\kappa$ and divide $R$ by $\kappa$ to obtain the same intervention. To address this issue, prior works on asymmetric low-rank optimization \citep{tu2016low,bhojanapalli2016dropping} typically incorporate a regularization term to close the norm gap between the scales of two low-rank matrices.
    In contrast to these methods, the optimization problem in equation~\eqref{eq:prob} does not require rescaling the two low-rank matrices $L(a)$ and $R$ due to the boundness property of the hyperbolic tangent function. These properties also help stabilize the training process.

    For a better understanding of the mapping constructed, we consider the degenerate case: if $W$ is a zero matrix and $\phi$ is a linear function, then $L(a) = b$ and $f(a) = (I + bR^\top) a + s$. This represents the classical bilinear low-rank mapping, which does not depend on the input vector $a$.

\section{Probe-Free Intervention Framework}
We propose a probe-free intervention framework based on the low-rank mapping described above. The framework performs the following tasks: (i) models the desirable answer region, and (ii) computes the intervention parameters by solving the associated smooth optimization problem. We begin by presenting the data matrix and the framework setup.

From the training data, we collect the activations of the desirable output and form a matrix $G_q \in \mathbb{R}^{d \times |\mathcal{G}(q)|}$ for each attention head, where $|\mathcal{G}(q)|$ denotes the cardinality of the set $\mathcal{G}(q)$. Each column of $G_q$ represents the activation of one desirable answer. We construct the undesirable matrix $B_q \in \mathbb{R}^{d \times |\mathcal{B}(q)|}$ using the same approach. Let $\mathcal{M}_q$ denote a generic manifold representing the region of desirable answer data points. The nonlinear low-rank mapping $f$ is trained by solving:

    \begin{equation} \label{eq:prob}
            \min_f~\sum_{q} \sum_{i \in \mathcal B(q) \cup \mathcal G(q)}~ 
            c_q ( f(a_i),  \mathrm{Proj}_{\mathcal M_q}(f(a_i)) ),
    \end{equation}
    where $c_q$ denotes the distance measure for question $q$, $f(a_i)$ denotes the terminal point of the intervention on activation $a_i$, and $\mathrm{Proj}_{\mathcal{M}_q}$ denotes the projection onto the manifold $\mathcal{M}_q$. The performance of the intervention depends on the choice of $c_q$ and the manifold $\mathcal{M}_q$. The manifold $\mathcal{M}_q$ should have a semantic interpretation in the modeling process and must be capable of separating the desirable answers from the undesirable answers. Additionally, solving problem~\eqref{eq:prob} should be computationally efficient.

We now describe one specific instance of our framework. The manifold $\mc M_q$ for question $q$ is chosen as an ellipsoid of the form
\begin{subequations} \label{eq:ellipsoid-all}
        \begin{equation} \label{eq:ellipsoid}
        \mathcal M_q = \{ x : (x-{\hat{\mu}}_q)^{\top} \hat{\Sigma}_q^{-1} (x-\hat{\mu}_q) \le \rho_q \},
        \end{equation} 
where $\hat \mu_q$ is the center of the ellipsoid, $\hat{\Sigma}_q$ is a positive definite matrix that prescribes the shape, or orientation, of the ellipsoid, and $\rho_q$ is the radius of the ellipsoid. 

Moreover, we choose $c_q$ as the squared Mahalanobis distance~\citep{de2000mahalanobis}
\begin{equation} \label{eq:ellipsoid-c}
    c_q(a, a') = (a - a')^{\top} \hat{\Sigma}_q^{-1} (a - a'),
\end{equation}
which is the distance of the candidate point from the center of mass divided by the width of the ellipsoid in the direction of the candidate point. This distance is rooted in the construction of the manifold \eqref{eq:ellipsoid}. We use the following projection:
\begin{equation} \label{eq:ellipsoid-proj}
\mathrm{Proj}_{\mc M_q}(y) = \arg \min_{x \in \mc M_q}~c_q(x, y).
\end{equation}
\end{subequations}

Theorem~\ref{thm:ellipsoid} gives the optimization problem under the Mahalanobis distance with analytical projection expression.
\begin{theorem}[Ellipsoidal manifold and Mahalanobis distance] \label{thm:ellipsoid}
Suppose that the manifold, the distance measure, and the projection operator are chosen as in~\eqref{eq:ellipsoid-all}.
Problem~\eqref{eq:prob} becomes
\begin{equation} \label{eq:prob-ellipsoid}
\min_f ~ \sum_{q} \sum_{i \in \mathcal B(q) \cup \mathcal G(q)}\!\!\left[ \left( \sqrt{c_q(f(a_i), \hat \mu_q)} - \sqrt{\rho_q}\right)_+ \right]^2,
\end{equation}
where $(y)_+ = \max\{0, y\}$. 
\end{theorem}

The objective function for the optimization problem is straightforward to compute, as it transforms the distance between $(I + L(a)R)a + s$ and its projection into the distance between $(I + L(a)R)a + s$ and a fixed vector $\hat{\mu}_q$. Readers can refer to the appendix \ref{sec:proof} for the proof of Theorem \ref{thm:ellipsoid}.

We now discuss how to derive the parameters to specify the manifold. The main challenge here is the high dimensionality of the activation vectors ($D = 4096$ for Llama3-8B), while for each question $q$, we observe fewer than ten samples in total \footnote{Data collected from the TruthfulQA dataset.}. Next, we present a practical method for estimating $\hat \mu_q$, $\hat \Sigma_q$, and the radius $\rho_q$. From the training data, we compute the question-specific mean vector for question $q$ and the mean vector computed by all questions
\begin{equation}
\begin{cases}
\hat{\mu}_q^+ = \frac{1}{| \mathcal G(q)|} \sum_{i = 1}^{|\mc G(q)|} (G_q)_i \\
 \hat{\mu}^+ = \frac{1}{N^+}\sum_{q}  \sum_{i = 1}^{|\mc G(q)|} (G_q)_i
\end{cases}
\label{eq:sample_mean_and_total_mean }
\end{equation}
where $N^+$ denotes the sample size of the desirable subset of data. Analogously, we compute $\hat \mu_q^-$ and $\hat \mu^-$ for the \textit{un}desirable subset.

Due to the noise incurred when computing $\hat{\mu}_q^+$ with small sample sizes, additional information must be incorporated to accurately construct the mean vector $\hat{\mu}_q$. We propose the following extrapolation scheme:
\begin{align}
&\hat{\mu}_{q} = \notag \\
&\hat{\mu}_{q}^{+} + \lambda \left(\alpha \underbrace{\frac{\hat{\mu}_{q}^{+}-\hat{\mu}_q^{-}}{\|\hat{\mu}_{q}^{+}-\hat{\mu}_q^{-} \|}}_{\text{question specific}} + (1-\alpha) \underbrace{\frac{\hat{\mu}^{+}-\hat{\mu}^{-}}{\|\hat{\mu}^{+}-\hat{\mu}^{-}\|}}_{\text{overall}} \right),
\label{eq:convex_combination}
\end{align}
where $\lambda > 0$ controls the magnitude of extrapolation. The extrapolation direction consists of a question-specific direction dictated by $\hat \mu_q^+ - \hat \mu_q^-$, and an overall direction dictated by $\hat \mu^+ - \hat \mu^-$. Both terms aim to guide $\hat{\mu}_{q}$ \textit{away from} the \textit{un}desirable values $\hat \mu_q^-$ and $\hat \mu^-$. The parameter $\alpha \in [0,1]$ controls the relative strength between the question-specific (local) and overall (global) directions. \citet{ref:li2024inference} proposed a translation of the activation along the truthful direction $\hat{\mu}^{+}-\hat{\mu}^{-}$, which coincides with the overall direction in our formula. The direction of extrapolation is illustrated in Figure~\ref {fig:extrapolation}, where the extrapolation vector is denoted by $\Vec{d}_{\textit{move}}$. 

\begin{figure}
    \centering
    \includegraphics[scale=0.3]{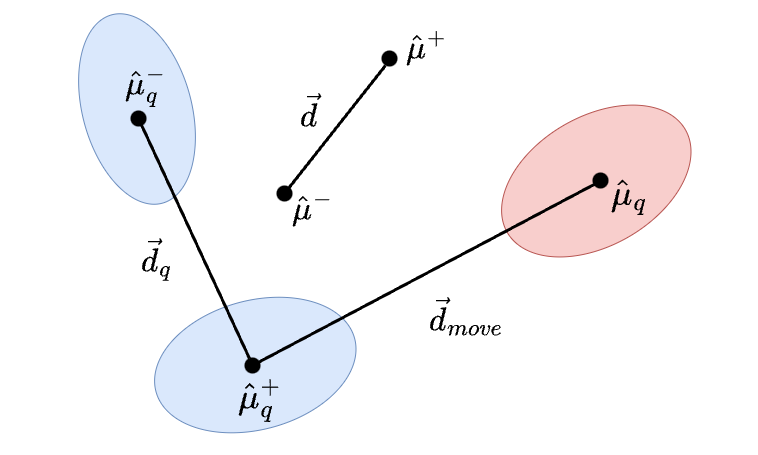}
    \caption{Exploration scheme in constructing question-wise mean vector $\hat{\mu}_{q}$, where the directions are computed as $ \Vec{d}=\hat{\mu}^{+}-\hat{\mu}^{-}$, $\Vec{d}_{q} = \hat{\mu}_{q}^{+}-\hat{\mu}_q^{-}$, and then take the scaled convex combination $\Vec{d}_{\textit{move}} = \hat{\mu}_{q}-\hat{\mu}_{q}^{+}  
    = \lambda \left(\alpha \frac{\Vec{d}}{\| \Vec{d}\|} + (1-\alpha) \frac{\Vec{d}_{q}}{\|\Vec{d}_{q}\|} \right)$.}
    \label{fig:extrapolation}
\end{figure}

For the covariance matrix, we assume that $\hat{\Sigma}_q^{-1}$ is constant across all $q$. This assumption is a fundamental component of linear discriminant analysis in machine learning \citep{tharwat2017linear}. We begin by computing the empirical covariance matrix:
    \begin{equation}
     S\!=\!\frac{1}{|N^{+}|-1} \sum_{q} \sum_{i=1}^{|\mc G(q)|} \! \big[ (G_q)_i- \hat \mu_q \big] \big[ (G_q)_i- \hat \mu_q \big]^\top.
     \label{eq:sample_cov}
    \end{equation}
    Due to small sample sizes, the empirical covariance matrix $S$ may be non-invertible. To address this issue, we adopt the linear shrinkage method~\citep{schafer2005shrinkage, ledoit2004well}, which shrinks the matrix towards a diagonal target, and set:
    \begin{equation*}
        \hat{\Sigma}_{q}^{-1} = (\beta S + (1-\beta) \mathrm{diag}(S))^{-1} \quad  \forall q,
    \end{equation*}
    where $\beta \in (0,1)$ is the shrinkage parameter. 
    
    To estimate radius $\rho$, the heuristic approach is to find the minimum $\rho$ such that all samples remain within the desirable answer region, centered at $\hat{\mu}_q^{+}$. For question $q$, we define:
    \begin{equation*}
\rho_q = \max\limits_{i \in |\mathcal{G}_{q}|}~[(G_q)_{i}-\hat{\mu}_{q}]^{\top}\hat{\Sigma}_{q}^{-1}[(G_q)_{i}-\hat{\mu}_{q}]. 
    \end{equation*} 
\section{Algorithms for Training Nonlinear Low-Rank Mapping}

The degenerated bilinear low-rank optimization problem has many applications in the statistical estimation field (e.g., matrix sensing, matrix recovery, PCA) \citep{chi2019nonconvex}. The literature shows that simple methods such as gradient descent can converge to the ground truth under mild assumptions with good initialization. In contrast, our objective function $ (W, R, b, s) \mapsto \mathcal{F}(W, R, b, s)$ cannot be represented by the factorization formulation and contains an extra nonlinear activation function, making it challenging to derive similar theoretical guarantees. However, the first-order algorithm is still a good choice with lightweight computational complexity. 

We emphasize that the activation function $\phi(\cdot) = \tanh(\cdot)$ introduces smoothness to the objective function, thereby stabilizing the optimization processes during gradient descent. Furthermore, we propose to use preconditioned gradient descent, a powerful approach that has gained popularity in (bilinear) low-rank optimization. This method effectively balances the scale at each iteration, accelerating convergence while maintaining the same computational complexity.
\begin{equation}
\begin{aligned}
& W_{t+1}\!=\!W_t\!-\!\eta \nabla_{W} \mathcal{F}\left(W_t, R_t,b_t,s_t\right)\left(R_t^{\top} R_t + \epsilon I\right)^{-1}, \\
& R_{t+1}\!=\!R_t \! - \! \eta \nabla_{R} \mathcal{F}\left(W_t, R_t,b_t,s_t\right)\left(W_t^{\top} W_t + \epsilon I\right)^{-1}, \\
& b_{t+1}\!=\!b_t-\eta \nabla_{b} \mathcal{F}\left(W_t, R_t,b_t,s_t\right)\left(b_t^{\top} b_t + \epsilon I\right)^{-1}, \\
& s_{t+1}\!=\!s_t-\eta \nabla_{s} \mathcal{F}\left(W_t, R_t,b_t,s_t\right)\left(s_t^{\top} s_t + \epsilon I\right)^{-1},
\end{aligned}    
\label{eq:alg-scaled-gd}
\end{equation}
where $\epsilon$ is a small number to ensure the scaling term is invertible. 
% This perturbation is also used in other LM papers related to Riemannian precondition \citep{zhang2024riemannian}. 
In comparison to vanilla gradient descent, the search directions of the variables in equation~\eqref{eq:alg-scaled-gd} are scaled. Intuitively, this scaling acts as a preconditioner, enhancing the search direction and enabling the use of larger step sizes. The preconditioner is adaptive and varies across iterations. Computationally, the scaled gradient descent introduces minimal overhead, as the inversion of the small-sized matrix is computationally inexpensive. Therefore, the per-iteration cost remains in the same order as standard gradient descent. In the context of bilinear low-rank matrix estimation, specifically in the problem $ \min_{b, R} f(bR^{\top}) $ \citep{Tong2021AcceleratingIL}, it has been shown that scaled gradient descent achieves a condition number-free convergence rate that outpaces standard gradient descent under mild assumptions.

\section{Empirical Results}

 This section presents the empirical results of our proposed algorithm FLORAIN. Section~\ref{sec:exp_truth_setup} clarifies the setting of our experiments on the TruthfulQA dataset, including details about datasets, tasks, metrics, baselines, and computational resources. Section~\ref{sec:exp_truth_results} showcases the superiority of our proposed methods to other baselines. We introduce the above contents in the Toxic Comments Classification Challenge dataset in Appendix \ref{sec:exp_toxic}. 
\subsection{Experimental Setup}
\label{sec:exp_truth_setup}
\textbf{Tasks and Metrics}: We evaluate our framework on the multiple choice and generation tasks, which are commonly used to verify the truthfulness of language models:
\begin{itemize}[leftmargin=5mm]
    \item In the generation task, the model produces a complete answer for each question using greedy autoregressive decoding. While the accuracy and informativeness of the answers are ideally assessed by humans, this evaluation method is both costly and time-consuming. As a result, most works in this area rely on a well-trained large language model as an alternative evaluation tool. In our approach, we employ two fine-tuned \texttt{GPT-3.5-instruct} models: one for classifying whether the answer is correct or incorrect, and another for determining whether the response is informative. According to \citet{ref:li2024inference}, the percentage of answers labeled as correct by the first model is referred to as the truthful score (True \%). Meanwhile, the percentage of answers labeled as informative by the second model corresponds to the informative score (Info \%). Following the methodology of \citet{ref:li2024inference}, we report both the truthful score (True \%) and the product of the truthful and informative scores, True*Info (\%).

    \item In the multiple-choice task, the model computes the log probability of completion for a given question and its corresponding set of choices. Following the approach of \citet{ref:lin2021truthfulqa}, we report two metrics: MC1 and MC2. MC1 identifies the correct answer as the one with the highest probability among the available choices, and the overall accuracy across all questions is denoted as MC1. MC2 represents the normalized total probability assigned to the set of true answers, given a question and its associated set of true/false reference answers.

    \item We also include two other metrics named Kullback-Leiber divergence (KL) of the model’s next-token prediction distribution post-versus-pre-intervention and Cross-Entropy Loss (CE). These two metrics assess the extent to which the generation distribution shifts after the intervention. Lower values are preferred, indicating that the intervention minimally alters the behavior of the original model, reducing the likelihood of producing abnormal characters or unnatural sentences. The calculation of these metrics is detailed in \citet{ref:li2024inference}.

\end{itemize}
\noindent \textbf{Datasets:} We use the TruthfulQA dataset to benchmark the effectiveness of FLORAIN. This dataset consists of 817 questions across 38 categories, designed to elicit false answers from language models, posing a challenge to generating accurate responses. Following the splits provided by \citet{ref:li2024inference} and \citet{yin2024lofit}, we divide the dataset into training (326 questions), validation (82 questions), and test (407 questions) sets. For evaluation, we perform 2-fold cross-validation, as done in \citet{ref:li2024inference} and \citet{yin2024lofit}. Additionally, we process the dataset into 5,918 question-answer pairs, each labeled with a binary indicator of desirability. The training set is used to develop our intervention policy, while the validation set is reserved for parameter tuning.

\noindent \textbf{Models:} We apply our method to multiple pre-trained models to enhance their truthfulness, including Llama-7B \citep{ref:touvron2023llama}, Llama2-chat-13B \citep{ref:touvron2023llama2}, and Llama3-8B \citep{dubey2024llama}. Our approach can be integrated into existing fine-tuning pipelines, helping to generate more accurate answers. To demonstrate its versatility, we also evaluate our method on models that have already been fine-tuned for the same task, treating these as both baseline models and competitors to our approach. Detailed descriptions of these models are provided in the baselines section.

\noindent \textbf{Hyperparameters}: We manipulate the behavior of FLORAIN by tuning two hyperparameters $\lambda$ and $\alpha$ in determining the $\hat{\mu}_q$ by equation \eqref{eq:convex_combination}. The details about selecting them for each pre-trained model are discussed in Appendix~\ref{sec:hyper}.

\noindent \textbf{Baselines:} We compare FLORAIN against competitive baselines with the same goal of eliciting truthful answers from language models:
\begin{itemize}[leftmargin=5mm]
    \item Inference-time Intervention (ITI, ~\citealt{ref:li2024inference}) is a state-of-the-art method that allows intervention without finetuning. We follow the hyperparameter settings provided in the original paper~\cite{ref:li2024inference} and their GitHub repository.\footnote{\url{https://github.com/likenneth/honest_llama/tree/master}}
    \item Few-shot prompting (FSP), introduced by \citet{bai2022training}, has demonstrated the effectiveness of prompting with 50 examples on the TruthfulQA benchmark. In our experiments, we use the 50-shot version, selecting 50 prompt pairs from the training set. 
    \item Instruction Fine-Tuning (IFT) \citep{wang2022self, chung2024scaling} is a well-known approach to improving the truthfulness of language models through fine-tuning. For comparison, we include two prominent models in this category: Alpaca-7B \citep{ref:taori2023alpaca} and Vicuna-7B \citep{ref:chiang2023vicuna}.
\end{itemize}

\noindent \textbf{Computation resources}. Our method and baselines run on four NVIDIA RTX A5000 GPUs, an i9 14900K CPU, and 128GB RAM.
\subsection{Numerical Results}
\label{sec:exp_truth_results}
\subsubsection{Comparison between Truthfulness Augmentation Baselines across Different LMs}

Table~\ref{table:benchmark-fine-tuning-free} presents the numerical results of our method, ITI, and Few-shot prompting (FSP) across three base models: Llama-7B, Llama3-8B, and Llama2-chat-13B on the TruthfulQA dataset. Notably, FSP is a prompt engineering technique orthogonal to intervention methods like FLORAIN and ITI. Therefore, we also report results for combinations of FSP with either ITI or FLORAIN. The FSP + FLORAIN combination consistently achieved the highest performance in metrics such as True * Info and True across all three models, with scores reaching 45\% for Llama-7B, 42\% for Llama3-8B, and 61\% for Llama2-chat-13B. 

When applied independently, FLORAIN outperforms ITI in four key metrics across both the generation and multiple-choice tasks, significantly enhancing the truthfulness and quality of the pre-intervention models. Notably, FLORAIN improves the True * Info(\%) score from 21.15 to 31.46 for Llama-7B, 32.88 to 36.78 for Llama3-8B, and 51.87 to 60.68 for Llama2-chat-13B. The KL values remain minimal across all models, indicating that the intervention effect is not overly aggressive.

\subsubsection{Comparison between ITI, FLORAIN, and Instruction Finetuning Methods.}

In this experiment, we benchmark two instruction fine-tuning methods, Alpaca and Vicuna, to assess whether our method further improves their quality and truthfulness. As shown in Table~\ref{table:benchmark-alpaca-vicuna}, FLORAIN demonstrates significant effectiveness, boosting the True * Info score by 8.07\% for Alpaca and 14.21\% for Vicuna. In both models, FLORAIN outperforms ITI across all quality metrics, highlighting its ability to enhance both accuracy and truthfulness.

\section{Conclusion and Future Directions}
In this paper, we introduced FLORAIN, a novel probe-free low-rank intervention framework for activation intervention. The framework aims to minimize the distance between the post-intervention vector and its projection onto the desirable answer region, an ellipsoid that is carefully estimated using first-order and second-order statistical information. Unlike existing intervention methods, our probe-free approach does not require classifiers to identify the responsible heads or layers for intervention. Instead, it formulates an optimization problem that incorporates region modeling and analytical projection. Additionally, the intervention is performed in a single layer of the output transformer network, allowing for efficient parallel computation and mitigating distribution shift phenomena. A potential direction for future research is to extend the region modeling approach, such as by considering the subspace spanned by the desirable answer matrix as the desirable answer region.

\section{Limitations}
Our method relies on the scale of the training dataset. In cases where the number of training samples is small, constructing the ellipsoid region becomes more challenging, as the estimated mean and covariance matrix are more susceptible to higher bias. 
The smooth optimization problem derived in the main text exhibits nonconvexity. Due to our limited understanding of its optimization landscape, the algorithm may converge to local minimizers. As a result, our first-order algorithm does not guarantee convergence to the global optimum. To solve this problem, we employ standard gradient descent without implementing preconditioning.

\textbf{Acknowledgments.} Viet Anh Nguyen gratefully acknowledges the generous support from the UGC Early Career Scheme Grant 24210924 and the CUHK’s Improvement on Competitiveness in Hiring New Faculties Funding Scheme.

\begin{table*}[htbp]
\caption{Quantitative results of different methods on TruthfulQA dataset, across different Language Models.}
\label{table:benchmark-fine-tuning-free}
\begin{subtable}[htbp]{\textwidth}
\begin{tabular}{lcccccc}
\toprule
Methods              & True * Info (\%) $\uparrow$ & True (\%) $\uparrow$ & MC1 $\uparrow$ & MC2 $\uparrow$ & \multicolumn{1}{l}{CE $\downarrow$} & KL $\downarrow$ \\ \midrule
Unintervened   & 21.15          & 22.16          & 25.58          & 40.54          & 2.13 & 0.00 \\
ITI            & 26.52          & 28.03          & 27.78          & 43.59          & 2.20 & 0.07 \\
FLORAIN (ours) & \textbf{31.46} & \textbf{34.72} & \textbf{31.76}          & \textbf{47.43}          & 2.21 & 0.09 \\ \midrule
FSP            & 36.13          & 39.78          & 34.03 & 50.34 & 2.13 & 0.00 \\
FSP + ITI      & 40.63          & 45.16          & 35.50          & 52.48          & 2.20 & 0.07 \\
FSP + FLORAIN (ours) & \textbf{45.31}              & \textbf{49.23}       & \textbf{36.45} & \textbf{54.27} & 2.20                                & 0.08            \\ \bottomrule
\end{tabular}
\caption{Llama-7B}
\end{subtable}
% \begin{table}[H]
\begin{subtable}[htbp]{\textwidth}
\begin{tabular}{lcccccc}
\toprule
Methods              & True * Info (\%) $\uparrow$ & True (\%) $\uparrow$ & MC1 $\uparrow$ & MC2 $\uparrow$ & \multicolumn{1}{l}{CE $\downarrow$} & KL $\downarrow$ \\ \midrule
Unintervened   & 32.88          & 44.18          & 30.36          & 48.98          & 2.38 & 0.00 \\
ITI            & 35.92          & 46.88          & 32.07          & 49.84          & 2.50 & 0.13 \\
FLORAIN (ours) & \textbf{36.78} & \textbf{48.67} & \textbf{34.56}          & \textbf{53.68} & 2.51 & 0.14 \\ \midrule
FSP            & 36.32          & 39.78          & 35.74 & 52.93          & 2.38 & 0.00 \\
FSP + ITI      & 40.63          & 45.16          & 35.50          & 52.98          & 2.48 & 0.14 \\
FSP + FLORAIN (ours) & \textbf{42.15}              & \textbf{47.32}       & \textbf{36.98} & \textbf{55.83} & 2.51                                & 0.16            \\ \bottomrule
\end{tabular}
\caption{Llama3-8B}
\end{subtable}
\begin{subtable}[htbp]{\textwidth}
\begin{tabular}{lcccccr}
\toprule
Methods & True * Info (\%) $\uparrow$ & True (\%) $\uparrow$ & MC1 $\uparrow$ & MC2 $\uparrow$ & \multicolumn{1}{l}{CE $\downarrow$} & \multicolumn{1}{c}{KL $\downarrow$} \\ \midrule
Unintervened         & 51.87          & 59.86          & 35.38          & 53.32          & 2.31 & 0.00 \\
ITI                  & 57.02          & 63.04          & 37.46          & 55.59          & 2.32 & 0.17 \\
FLORAIN (ours)       & \textbf{60.68} & \textbf{67.70} & \textbf{39.65}          & \textbf{59.57} & 2.35 & 0.18 \\ \midrule
FSP                  & 55.97          & 58.63          & 40.76 & 57.84 & 2.31 & 0.00 \\
FSP + ITI            & 56.78          & 59.24          & 41.50 & 59.01          & 2.33 & 0.13 \\
FSP + FLORAIN (ours) & \textbf{61.14} & \textbf{62.45} & \textbf{44.52} & \textbf{61.48} & 2.37 & 0.16 \\ \bottomrule
\end{tabular}
\caption{Llama2-chat-13B}
\end{subtable}
\end{table*}
\begin{table*}[htbp]
\caption{Quantitative results of intervention methods on instruction-finetuned models Alpaca and Vicuna.}
\label{table:benchmark-alpaca-vicuna}
\begin{tabular}{lcccccc}
\toprule
Methods & True*Info (\%) $\uparrow$ & True (\%) $\uparrow$ & MC1 $\uparrow$ & MC2 $\uparrow$ & \multicolumn{1}{c}{CE $\downarrow$} & \multicolumn{1}{c}{KL $\downarrow$} \\ \midrule
Alpaca                  & 30.39          & 30.85          & 26.56          & 41.63          & 2.81 & 0.00 \\
Alpaca + ITI            & 37.67          & 38.19          & 28.89          & 45.19          & 2.88 & 0.14 \\
Alpaca + FLORAIN (ours) & \textbf{38.56} & \textbf{40.52} & \textbf{31.32} & \textbf{46.83} & 2.85 & 0.15 \\ \midrule
Vicuna                  & 38.24          & 42.10          & 31.83          & 48.48          & 2.67 & 0.00 \\
Vicuna + ITI            & 49.27          & 53.25          & 33.42          & 51.80          & 2.77 & 0.26 \\
Vicuna + FLORAIN (ours) & \textbf{52.45} & \textbf{57.81} & \textbf{36.86} & \textbf{56.76} & 2.78 & 0.27 \\
\bottomrule
\end{tabular}
\end{table*}

\clearpage
\bibliography{intervention}

\clearpage
\appendix
\section{Hyper-parameter Tuning}
\label{sec:hyper}
In our framework, two key hyperparameters are $\lambda$ and $\alpha$. We determine their values by evaluating performance metrics and text generation quality in the validation set. Through a grid search over $\{2, 3, 4, 5\}$ for $\lambda$ and $\{0.0, 0.2, 0.4, 0.6, 0.8, 1.0\}$ for $\alpha$, we find that setting $\alpha = 0.2$ and $\lambda = 5$ yields strong performance across various cases. As a result, unless otherwise specified, we adopt this combination for all our experiments.

Additionally, the choice of the intervened layer is crucial. Given the limited number of layers, we can efficiently conduct a search procedure to optimize validation test performance. Specifically, we intervene at layer 11 for Llama-7B, Alpaca-7B, Vicuna-7B, layer 12 for Llama3-8B, and layer 14 for Llama2-chat-13B.

\section{Proof of Theorem~\ref{thm:ellipsoid}}
\label{sec:proof}
The intervention maps the activation of undesirable answer $a_i$, $i \in \mathcal B(q)$ onto the ellipsoid region of the desirable answers, which can be defined below.
    \begin{equation}
        \mathcal E_q = \{ x : (x-\hat{\mu}_q)^{\top} \hat{\Sigma}_q^{-1} (x-\hat{\mu}_q) \le \rho_q \}
    \end{equation} 
    % Define the matrix X
In the proof part, we simplify the notation of $\hat{\Sigma}_q$ (resp. $\mathcal E_q, \hat \mu_q$) as $\Sigma$ (resp. $\mathcal{E}, \mu$) for clarity. Let the Mahalanobis projection be
\[
\mathrm{Proj}_{\mathcal E}(y) = \arg\min_{x \in \mathcal E}~ \left(y-x\right)^{\top} \Sigma^{-1} \left(y-x\right),
\]
where 
\[
   \mathcal E = \{ x : (x-\mu)^{\top} {\Sigma}^{-1} (x-{\mu}) \le \rho \}.
\]

\begin{proof}[Proof of Theorem \ref{thm:ellipsoid}]
If $y \in \mathcal{E}$, the result can be immediately derived by the definition of the projection. If $y \notin \mathcal{E}$, we can transform the coordinates of $x,y$ to simplify the projection problem. We define
\[
   \hat{y} \Let \Sigma^{-\frac{1}{2}} y, \quad 
   \hat{x} \Let \Sigma^{-\frac{1}{2}} x.
\]
The projection problem has been transformed into
\[
\arg\min_{\hat{x} \in \mathcal{\hat{E}} } \|\hat{y}-\hat{x} \|_2^{2},
\]
where
$\mathcal{\hat{E}} = \{ \hat{x} : \|\hat{x}-\Sigma^{-\frac{1}{2}}\mu\|_2^{2} \le \rho \}$.

The transformed problem allows for a closed-form solution for the point outside of the region
\begin{align*}
    \hat{x}^{*} &= \Sigma^{-\frac{1}{2}}\mu + \sqrt{\rho} \frac{\hat{y}-\Sigma^{-\frac{1}{2}}\mu}{\|\hat{y}-\Sigma^{-\frac{1}{2}}\mu \|_2} \\
    & = \Sigma^{-\frac{1}{2}}\mu + \sqrt{\rho} \frac{\Sigma^{-\frac{1}{2}} y-\Sigma^{-\frac{1}{2}}\mu}{\|\Sigma^{-\frac{1}{2}} y-\Sigma^{-\frac{1}{2}}\mu \|_2}.
\end{align*}
Therefore 
\begin{align*}
    x^{*} & = \mu + \sqrt{\rho} \frac{ y-\mu}{\|\Sigma^{-\frac{1}{2}} y-\Sigma^{-\frac{1}{2}}\mu \|} \\
    & = \mu +  \frac{ \sqrt{\rho}}{\sqrt{(y -  \mu)^{\top} {\Sigma}^{-1} (y - \mu)}} \left( y-\mu\right).
\end{align*}
   For an arbitrary $q$ and $i \in \mathcal B(q) \cup \mathcal G(q)$, 
    \begin{align*}
    & f(a_i) - \mathrm{Proj}(f(a_i))  \\ & = f(a_i) - \hat{\mu}_{q} \\ & - \frac{\sqrt{\rho_q}}{\sqrt{(f(a_i) -  \hat{\mu}_{q})^{\top} \hat{\Sigma}^{-1} (f(a_i) - \hat{\mu}_{q})}} (f(a_i) - \hat{\mu}_{q}) \\
    &= \Big[ 1 - \frac{\sqrt{\rho_q}}{\sqrt{(f(a_i) -  \hat{\mu}_{q})^{\top} \hat{\Sigma}^{-1} (f(a_i) - \hat{\mu}_{q})}} \Big] (f(a_i) - \hat{\mu}_{q}).
    \end{align*}
    By computation, if $f(a_i) \in \mathcal M_q$, then
    \[
    c_q(f(a_i), \mathrm{Proj}_{c_q, \mathcal M_q}(f(a_i))=0.
    \]
    Otherwise, the distance is
    \[
[\sqrt{(f(a_i) -  \hat{\mu}_{q})^{\top} \hat{\Sigma}^{-1} (f(a_i) - \hat{\mu}_q)} - \sqrt{\rho}]^2.
    \]
    We complete the proof.
\end{proof}

\section{Experimental results on the toxicity mitigation task}
\label{sec:exp_toxic}
\subsection{Experimetal Setup}
\noindent \textbf{Tasks:}
In this task, language models must complete an incomplete prefix of text. Typically, the prefix is chosen to provoke toxic content from LLMs. We evaluate the models based on three key metrics: toxicity, fluency, and diversity. For each prompt in the dataset, the models generate 25 responses, each limited to 20 tokens. These outputs are analyzed using the Perspective API \footnote{\url{https://perspectiveapi.com/}}, which estimates the likelihood that a human would perceive the text as toxic.

\noindent \textbf{Metrics:}
\begin{itemize}
    \item Expected Maximum Toxicity (Exp. Max. Tox.): We identify the highest toxicity score for every output and compute the average of these maximum scores across all prompts.
    \item Toxic Completion Proportion (Tox. Prob.): This metric tracks the fraction of outputs considered toxic, where toxicity is defined as a score above 0.5 based on the threshold of Perspective API. 
    \item Fluency is evaluated by calculating the perplexity values of the generated outputs, using GPT-2 (XL) as a reference model. Lower perplexity values suggest that the text is more coherent and fluent.
    \item Diversity is assessed by examining the ratio of unique n-grams (1-gram, 2-gram, and 3-gram) to the total number of tokens in the generated text. This metric captures the range of variation in the outputs, with higher values indicating more diverse and varied language use.
\end{itemize}

\noindent \textbf{Training Dataset:}
We use the Toxic Comments Classification Challenge data.\footnote{\url{https://www.kaggle.com/c/jigsaw-toxic-comment-classification-challenge}} The dataset comprises sentences and their human toxicity labels. Following existing works in the field, we adopt GPT2-Large as the base model across all experiments of the toxicity mitigation task. We include several baselines that have the same goal of reducing the toxicity of large language models (LLMs). As for MIMIC, we consider two versions: Mean Matching (MM) and Mean+Covariance Matching (MCM).
\subsection{Comparison}
The experimental results of baselines are shown in Table \ref{table:toxic}, where the base model used by all methods is GPT-2 Large. The result of the original model is described in the first row. We split baselines into two groups. The first one using an extensive finetuning procedure comprises DAPT, GeDI, PPLM, UDDIA, DExperts, and GOODTRIEVER, while the second group contains inference time finetuning-free methods like MIMIC, ITI, and FLORAIN. Baselines in the first group are better than counterparts in the second group regarding toxicity metrics. However, these methods necessitate either fine-tuning or computing gradients at inference time, which can be computationally intensive. MIMIC, ITI, and FLORAIN achieved a toxicity reduction comparable to that of many algorithms in the first group. Specifically, FLORAIN is superior to PPLM and equally competitive to DAPT. Notably, within the second group, FLORAIN offers the best toxicity reduction impact than ITI and MIMIC while maintaining a better fluency and diversity of generated sentences. The fluency of FLORAIN is even more favored than almost all algorithms in the first group except for UDDIA. At the same time, its diversity metric is better than that of other baselines apart from PPLM.

\begin{table*}[htbp]
\caption{Quantitative results of our method and baselines on toxicity mitigation task.}
\label{table:toxic}
\begin{tabular}{lccccccccc}
\toprule
Model & Exp. Max. Tox. $\downarrow$ & Tox. Prob. $\downarrow$ & Fluency $\downarrow$ & 1-gram $\uparrow$ & 2-gram $\uparrow$ & 3-gram $\uparrow$ \\ \midrule
GPT-2 (large)      & 0.39 & 0.25 & 24.66 & 0.58 & 0.85 & 0.85 \\
DAPT               & 0.27 & 0.09 & 30.27 & 0.57 & 0.84 & 0.84 \\
GeDI               & 0.24 & 0.06 & 48.12 & 0.62 & 0.84 & 0.83 \\
PPLM (10\%)        & 0.38 & 0.24 & 32.58 & 0.58 & 0.86 & 0.86 \\
UDDIA              & 0.24 & 0.04 & 26.83 & 0.51 & 0.80 & 0.83 \\
DExperts           & 0.21 & 0.02 & 27.15 & 0.56 & 0.84 & 0.84 \\
GOODTRIEVER        & 0.22 & 0.04 & 27.11 & 0.58 & 0.82 & 0.83 \\
MM (MIMIC)         & 0.33 & 0.16 & 28.00 & 0.58 & 0.85 & 0.85 \\
MCM (MIMIC)        & 0.29 & 0.09 & 30.70 & 0.54 & 0.84 & 0.84 \\
ITI                & 0.31 & 0.12 & 33.12 & 0.57 & 0.85 & 0.85 \\
\midrule
FLORAIN            & 0.27 & 0.10 & 28.11 & 0.58 & 0.85 & 0.85 \\
\bottomrule
\end{tabular}
\end{table*}
\end{document}